\documentclass{article}

\PassOptionsToPackage{square,numbers}{natbib}
\usepackage[nonatbib, final]{neurips_2020}

\usepackage[utf8]{inputenc} % allow utf-8 input
\usepackage[T1]{fontenc}    % use 8-bit T1 fonts
\usepackage{url}            % simple URL typesetting
\usepackage{booktabs}       % professional-quality tables
\usepackage{nicefrac}       % compact symbols for 1/2, etc.
\usepackage{microtype}      % microtypography

\usepackage{graphicx}
\usepackage{subfigure}
\usepackage{booktabs} % for professional tables
\usepackage{tikz}
\usepackage{pgfplots}
\usepgfplotslibrary{fillbetween}
\usepackage{float}
\usepackage{amsmath}
\usepackage{amssymb}
\usepackage{amsthm}
\usepackage{bbm}
\usepackage{amsthm}
\usepackage{enumitem}
\usepackage{xcolor}
\usepackage{dsfont}
\usepackage{geometry}
\usepackage{mathtools}
\usepackage{multirow}
\usepackage{multicol}
\usepackage{blkarray}

\let\temp\rmdefault
\usepackage[bb=pazo,frak=esstix]{mathalfa}
\let\rmdefault\temp

\def\+#1{\mathcal{#1}}
\newcommand{\E}{{\mathbb E}}

\newcommand{\T}{{\mathsf T}}
\renewcommand{\P}{{\mathbb P}}
\newcommand{\1}{{\mathbb 1}}

\newcommand{\argmin}{{\operatorname{argmin}}}

\newcommand{\eqdef}{\overset{\mathrm{def}}{=\joinrel=}}
\DeclarePairedDelimiterX{\infdivx}[2]{(}{)}{%
  #1\;\delimsize\|\;#2%
}
\newcommand{\KL}{\operatorname{KL}\infdivx}

\let\emptyset\varnothing

\theoremstyle{definition}
\newtheorem{definition}{Definition}[section]
\newtheorem{example}[definition]{Example}
\newtheorem{remark}[definition]{Remark}

\newtheorem{theorem}[definition]{Theorem}
\newtheorem{proposition}[definition]{Proposition}

\newtheorem{lemma}[definition]{Lemma}

\usepackage{natbib}
\usepackage{hyperref}
\hypersetup{
    colorlinks,
    linkcolor=[RGB]{165,88,88},
    citecolor=[RGB]{165,88,88},
    urlcolor=[RGB]{51,102,187}
}
\usepackage{cleveref}
\usepackage{soul}

\definecolor{sencha}{RGB}{225,25,25}

\begin{document}

% \kaifucomment{5/26 Meeting:
% \begin{itemize}
%     \item Discuss the abstract and title.
%     \item There is a gap in the proof of "averaging" subsection I cannot fix. Currently I have deleted this part. It won't affect on our main contribution. I can move it to future work. 
%     \item But it will affect how we discuss the NER/PoS problem. Without averaging, this supervision can be more complicated or problematic. For an intuitive understanding, recall that the conditional mutual information between PoS and NER is zero.
%     \item Probably we can consider another (practical) example. We do not need a lot of details. Maybe the QA problem (or response-driven learning) where $X$ is question, $Y$ is sematic parsing and $O$ is answer, is more suitable. It would be better if $O|X$ is random.
%     \item Overview of the current paper structure. 
%     \item TODO: a practical example, citation format, introduction and conclusion.
%     \item Meet again this week?
% \end{itemize}
% }

\title{Learnability with Indirect Supervision Signals}
% \kaifucomment{Should we emphasize "framework"? For example: A Theoretical Framework of Learning with Indirect Supervision Signals}

\date{}

\author{%
  Kaifu Wang \\
  % Applied Mathematics and Computational Science \\
  University of Pennsylvania \\
  % Philadelphia, PA 19104 \\
  \texttt{kaifu@sas.upenn.edu} \\
  % examples of more authors
  \And
  Qiang Ning\thanks{Work done while at the Allen Institute for AI and at the University of Illinois at Urbana-Champaign.} \\
  Amazon \\
  % Address \\
  \texttt{qning@amazon.com} \\
  \And
  Dan Roth \\
  University of Pennsylvania \\
  % Address \\
  \texttt{danroth@seas.upenn.edu} \\
  % \And
  % Coauthor \\
  % Affiliation \\
  % Address \\
  % \texttt{email} \\
  % \And
  % Coauthor \\
  % Affiliation \\
  % Address \\
  % \texttt{email} \\
}

\maketitle
\begin{abstract}
Learning from indirect supervision signals is important in real-world AI applications when, often, gold labels are missing or too costly. In this paper, we develop a unified theoretical framework for multi-class classification when the supervision is provided by a variable that contains nonzero mutual information with the gold label. The nature of this problem is determined by (i) the transition probability from the gold labels to the indirect supervision variables and (ii) the learner's prior knowledge about the transition. Our framework relaxes assumptions made in the literature, and supports learning with unknown, non-invertible and instance-dependent transitions. Our theory introduces a novel concept called \emph{separation}, which characterizes the learnability and generalization bounds. We also demonstrate the application of our framework via concrete novel results in a variety of learning scenarios such as learning with superset annotations and joint supervision signals. 
\end{abstract}

\section{Introduction}

% Standard supervised machine learning requires a set of training samples of an instance variable $X$ and its label $Y$.
We are interested in the problem of multiclass classification where direct and gold annotations for the unlabeled instance are expensive or inaccessible, and instead the observation of a 
dependent variable of the true label is used as supervision signal.
Examples include learning from
noisy annotations \citep{AngluinLa88, Kearns98, NDRT13}, 
partial annotations \citep{CourSaTa11, LiuDi14, Cid-SueiroGaSa14} %annotation 
or 
feedback from an external world 
\citep{CGCR10, BCFL13}.

To extract the information contained in a dependent variable, the learner should have certain \emph{prior knowledge} about the \emph{relation} between the true label and the supervision signal, which can be expressed in various forms.
For example, in the noisy label problem, the noisy rate is assumed to be bounded by a constant %constants 
(such as the \emph{Massart noise} \citep{MassartNe06, DiakonikolasGoTz19}). %
% \dr{in the bib items, you need to put things like PAC in curly brackets \{\} to keep them upper case.}
%
In the superset problem, the true label is commonly assumed to be contained in (or consistent with) the superset annotation \citep{CourSaTa11, LiuDi14}.

As in \citep{Cid-Sueiro12, RFDL16, RooyenWi18}, we model the aforementioned relation using a \emph{transition probability}, which is the distribution of the observable variable conditioned on the label and instance. 
The transition enables the learner to induce a prediction of the observable via the prediction of the label, and construct loss functions based on the induced prediction and the observable.

In this paper, instead of assuming that the learner \emph{fully knows} the transition, we formalize the concept of \emph{transition class}, a set that contains all the candidate transitions, to describe more general forms of prior information.
Also, we define the concept of \emph{separation} to quantify whether the information is enough to distinguish different labels.
With these concepts, we are able to study a variety of learning scenarios with unknown, non-invertible and instance-dependent transitions in a unified way. 
% \qn{we are able to study ... in a unified way, under the condition that the hypothesis space contains the true classifier, which is a common assumption in the literature \citep{Awasthi15b-noise,GHOSH201593-noise,superset-theory}.}
% \dr{I actually like the phrasing in black. But instead of the next sentence you can say:We show this under the realizablity assumption, a commonly made assumption (REF) that assumes that the true classifier is in the hypothesis space} \qn{I agree that sounds better} 
We show this under the \emph{realizability assumption} (also called \emph{separable} in linear classification), a commonly made assumption (such as \citep{ABHU15, GhoshMaSa15, LiuDi14}) that assumes that the true classifier is in the hypothesis space.

Our goal is to develop a \emph{unified theoretical framework} that can (i) provide learnability conditions for general indirect supervision problems, (ii) describe 
%In particular, describe 
what prior knowledge is needed about the transition, and (iii) characterize the difficulty of learning with indirect supervision. 

% \dr{I am thinking on whether there is a need to include here, before the detailed contribution, a ``simpler" statement about the key idea of the paper; e.g. something like: (Feel free to dismiss if you think that this is too descriptive or not accurate) ``Simply stated, the key contribution of the paper is in rigorously defining the concept of ``separation" and using it to provide necessary and sufficient conditions for learnability from indirection supervision -- the prior knowledge about the relation between the true label and the observed variable needs to be sufficient to ``separate" different labels via the observable."(Then connect to the rest: specifically, the main contribution of these paper are:)}
% \kaifucomment{I will try to write one.}

Specifically, in this paper, our main contribution includes:
\begin{enumerate}[leftmargin=*]
    \item We decompose the learnability condition of a general indirect supervision problem into three aspects: \emph{complexity}, \emph{consistency} and \emph{identifiability} and provide %provides 
    a unified learning bound for the problem (Theorem \ref{learnability-conditions}).
    \item We propose a simple yet powerful concept called \emph{separation}, which encodes the prior knowledge about the transition using statistical distance between distributions over the annotation space and uses it to characterize \emph{consistency} and \emph{identifiability} 
    % bridges the gap between theorem \ref{learnability-conditions} and practical problems
    (Theorem \ref{separation}).
    \item We formalize two ways to achieve separation: total variation and joint supervision, and use them %which are used 
    to derive concrete novel results of practical learning problems of interest (Section \ref{total-variation} and \ref{joint-supervision}).
\end{enumerate}

All proofs of the theoretical results are presented in the supplementary material.
\section{Related Work}

\paragraph{Specific Indirect Supervision Problems.} 

Our work is motivated by many previous studies on the problem of learning in the absence of gold labels. 
Specially, the problem of classification under \emph{label noise} dates back to \citep{AngluinLa88} and has been studied %is 
extensively %studied 
over the past decades.
Our work is mostly related to 
(i) Theoretical analysis of PAC guarantees and consistency of loss functions, including
learning with bounded noise
\citep{MassartNe06, Kearns98, ABHU15},
and instance-dependent noise
\citep{RalaivolaDeMa06, MenonRoNa18, CLRT17}.
(ii) Algorithms for 
learning from noisy labels, 
including using the inverse information of the transition
\citep{NDRT13, RooyenWi18},
and inducing predictions of noisy label (which is more similar to our formulation)
\citep{BootkrajangKa12, SBPBF14}.

% \dr{Was not sure how you wanted to link it to the previous paragraph}
% \kaifucomment{I will point out the superset can be regarded as a generalization of the noise label, but the way of theoretical analysis is different in literature.}
\emph{Superset} (also called \emph{partial label}) problems, where the annotation is given as a subset of the annotation space, arises in various forms in standard multiclass classification and structured prediction
\citep{CourSaTa11,
Cid-SueiroGaSa14,
INHS17,
NHFR19}.
While it is possible to extend some approaches in the theory of noisy problems to the superset case, the superset problem focuses on the case of a large and complex annotation space, and some of the assumptions (such as ``known transition") would be too strong in practice.
On the theoretical side, \citep{CourSaTa11} defines \emph{ambiguity degree} to characterize the learning bound. 
\citep{LiuDi14} provides an insightful discussion of the PAC-learnability of the superset problem and proposes the concept of \emph{induced hypothesis}. 
This two papers motivate the approach pursued in this paper.

\paragraph{Frameworks for Indirect Supervision.} 
Our supervision scheme is conceptually similar to 
\citep{SteinhardtLi15, RFDL16},
which model the label as a latent variable of the indirect supervision signal. However, the discussion is restricted to the exponential family model.
\cite{CGRS10,CSGR10} also propose a framework and algorithms to supervise a structured learning problem with an indirect supervision, which is modeled as a binary random variable associated with the gold label. Our work extends the binary indirect signal to a multiclass signal and gives a theoretical treatment to this general learning problem.
\citep{Cid-Sueiro12, Cid-SueiroGaSa14} 
study the problem of designing consistent loss functions for superset problems when the transition (aka \emph{mixing}) matrix is \emph{partially known}. The discussion can be applied to a wider range of problems such as noisy and semi-supervised learning. 
Our framework can also be compared to the multitask learning framework proposed in \citep{Ben-DavidBo07, Ben-DavidBo03}, which defines a notion called the $\+F$-relatedness to describe the relation machine learning tasks through some deterministic functions within the instance space $\+X$. As contrast, our framework studies the probabilistic transitions between different domains (from the label space $\+Y$ to the annotation space $\+O$) and our gold label $Y$ is not observable, which is different than a multitask setting.
Our goal is mostly related to \citep{RooyenWi18}, which further develops the ideas from \citep{SBHy13, 
Cid-SueiroGaSa14,
NDRT13} 
and develops a general framework of learning from data with \emph{reconstructible} corruption, using the inverse of a known, instance-independent transition matrix, to construct unbiased estimator of the classification loss and derive generalization bounds. Our study aims to relax these assumption on the transition matrix, especially when the label space and/or annotation space is large and the estimation of the transition matrix could be difficult.
% \dr{I think that we should cite here also \cite{CGRS10,CSGR10}}
% \kaifucomment{Updated (third line of this paragraph). Please take a look.}

% \paragraph{Complexity Analysis.} Our study also benefits from notations of complexity for hypothesis class, including classical VC-dimension \citep{VC-dimension}, Natarajan dimension \citep{Natarajan} and Rademacher complexity \citep{Rademacher-complexity}. Specially, we will use \emph{weak-VC major} introduced in \citep{weak-vc} to bridge the gap between Natarajan dimension and Rademacher complexity for indirect supervision. 
\section{Preliminaries} \label{notation}

We will use $\P(\cdot)$ to denote probability, $\E[\cdot]$ to denote expectation, $\1\{\cdot\}$ to denote the indicator function and $p(\cdot)$ to denote the density function or more generally, the Radon–Nikodym derivative.

We denote the source variable as $X$, which takes value in an input space $\mathcal{X}$ and denote the target label as $Y$, which takes value in a label space $\mathcal{Y}$. We assume $|\mathcal{Y}|=c$ is finite and identify the elements in $\mathcal{Y}$ as $\{y_1,y_2,\dots,y_c\}$. The goal is to learn a mapping $h_0:\mathcal{X}\rightarrow \mathcal{Y}$. The \emph{hypothesis class} $\mathcal{H}$ contains candidate mappings $h:\mathcal{X}\mapsto\mathcal{Y}$. The loss function for hypothesis $h$ and sample $(x,y)$ is denoted as $\ell(h(x),y)$. The \emph{risk} of a hypothesis $h\in\mathcal H$ is defined as $R(h):=\E_{X,Y}[\ell(h(x),y)]$
% \qn{$\E_{X,Y}[\ell(h(x),y)]$?}
% \kaifucomment{I think both are acceptable}
, where $x$ is sampled independently from a (unknown) distribution $D_X$.
We will focus on the realizable case, i.e., there is a classifier $h_0\in\mathcal{H}$ such that $R(h)=0$. As in standard PAC-learning theory, we use the zero-one loss for the gold sample $(x,y)$ (although we may not observe $y$): $\ell(h(x),y))=\1\left\{h(x)\ne y\right\}$.

An \emph{annotation} $O$ (also called \emph{supervision signal}) is a random variable that is not independent with $Y$ (or equivalently, $O$ and $Y$ has positive mutual information). 
The dependence between $X$ and $O$ conditioned on $Y$ is allowed but not required. $O$ takes value in an annotation space denoted as $\mathcal O$. We also assume $|\mathcal{O}|=s<\infty$ and identify the elements in $\mathcal{O}$ as $\{o_1,o_2,\dots,o_s\}$. For convenience, when using $y_i$ and $o_i$ as subscripts, we regard $y_i,o_i$ as its index $i$. For example, for any indexed quantity $a_i$,  we will denote 
$a_{y_i}$ as $a_i$. 
% \qn{what's $a$ then?}
We denote the probability simplex of dimension $s$ as: $\+D_\+O=\{w\in\mathbb{R}^s:\sum_{i=1}^sw_i=1,w_i\ge0\}$, which represents the set of all distributions over $\+O$.

Examples of annotation $O$: 
(i) In the \emph{noisy problem}, the true label is replaced (due to mislabeling or corruption) by another label with 
% \qn{with some or certain probability?}
certain probabilities. Therefore, $\mathcal O=\mathcal Y$.
(ii) In the \emph{superset annotation} problem, the learner observes $o$ which is a subset of $\mathcal Y$ (hopefully but not necessarily $o$ contains the true label $y$). In this case, $\mathcal{O}=2^\mathcal{Y}$, the power set of $\+Y$. 
% \qn{do we want to give an example of $\+D_\+O$ as well?} \kaifucomment{Given above}
% (iii) In a \emph{response driven learning} problem \citep{world-response}, the learner maps a sentence $x$ to a semantic parsing $y$, and $y$ can be executed .
% \kaifucomment{add an example here.}

In our framework, the learner predicts $O$ using the graphical model shown in \cref{fig:annotation-graph}.
The conditional distribution of $O$ given $X=x$ and $Y=y$ can be identified by a mapping from $x$ to a \emph{transition matrix} $T_0(x):=[\P(O=o_j|X=x, Y=y_i)]_{ij}$. 
% \kaifucomment{I feel it is strange to add a citation here. In previous work when people use concepts such as transition/confusion/mixing matrix, they simply define it without citing anything. We can also do the same.}
A key point of this paper is that in general we do \emph{not} assume that the learner (or learning algorithm) has full information of $T_0(x)$. 
Instead, we define a \emph{transition hypothesis} to be a candidate transition $T(x)$ (also denoted as $T$ for convenience) that maps the instance $x$ to a stochastic matrix of size $c\times s$.
For a fixed $x$, the $i^\text{th}$ row of a transition $T(x)$ represents a distribution over $\+O$, and is denoted as $(T(x))_i$.
The set of all candidate transition hypotheses is called the \emph{transition class}, denoted as $\mathcal{T}$.
We assume $T_0\in \mathcal{T}$.
% The transition class does not need to be defined explicitly, instead, it can be defined by constraints. For example, if one assumes $O$ is independent of $X$, then we have $T$ is contained 
When it is needed to distinguish transition hypothesis from classifiers in $\mathcal{H}$, we will call the latter one a \emph{base hypothesis}.
With a transition hypothesis $T$, a base hypothesis $\hat y=h(x)$ naturally induces a probability distribution $(T(x))_{h(x)}$. We call it \emph{induced hypothesis}, denoted as $T\circ h$.

% Figure 1: supervision model
\begin{figure}[H]
    \centering
    \tikzstyle{dir}=[->, line width=0. 2mm]
    \tikzstyle{bi}=[<->, line width=0. 2mm]
    \begin{tikzpicture}[auto]
        \node (left) at (-2,0) {$X$};
        \node (middle) at (0,0) {$\widehat Y$};
        \node (right) at (2.6,0) {$\widehat{\P}(O|X,\widehat{Y})$};
        % \node (data) at (4,0) {$O$};

        \node at (-2.7 ,0) {source};
        \node at (0,0.5) {label};
        \node at (4.3,0) {annotation};
        % \node at (4,0) {data};

        \draw[dir] (left) .. controls (-0.8,-0.5) and (0.9,-0.5) .. (right);
        \draw[dir] (left) to [above] node {$\+H$} (middle);
        \draw[dir] (middle) to [above] node {$\+T$} (right);
        \draw[dir] (middle) to [above] node {$\+T$} (right);
        % \draw[bi] (data) to [above] node {$\ell_\+O$} (right);

    \end{tikzpicture}  
    \caption{Supervision Model. The learner predicts the label $\widehat{Y}$ of $X$ via $\+H$. To supervise using the observation of $(X,O)$, the learner uses $\widehat Y$ to induce a probabilistic prediction over $\+O$ via $\+T$.}
    \label{fig:annotation-graph}
\end{figure}
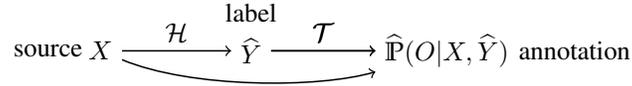

One may penalize $T\circ h$ by evaluating its prediction of $O$ on the dataset. More precisely, in our framework, the learner will be penalized by provided with an \emph{annotation loss} $\ell_\mathcal{O}(\widehat y,T,(x,o)):\mathcal{Y} \times \mathcal{T}\times \mathcal{X} \times \mathcal{O}\mapsto \mathbb{R}$. A natural example is the cross-entropy loss, which approximates the conditional probability of $O$: 
\begin{equation} \label{cross-entropy}
    \ell_\mathcal{O}(h(x),T,(x,o)):=-\log \P(o|x,h(x),T)
\end{equation}
The \emph{annotation risk} is defined as $R_\mathcal{O}(T\circ h):=\E_{x,o}[\ell_\mathcal{O}(h(x),T,(x,o))]$.
% \qn{$R_\mathcal{O}(h\circ T)$ or $R_\mathcal{O}(T\circ h)$}. 
A \emph{training set} $S=\{(x^{(i)},o^{(i)})\}_{i=1}^m$ contains independent samples of $X$ and $O$.
The \emph{empirical annotation risk} associated with the training sample $S$ is then defined as $\widehat R_\mathcal{O}(h\circ T|S):=\frac{1}{m}\sum_{i=1}^m\ell_\mathcal{O}(h(x^{(i)}),T,(x^{(i)},o^{(i)}))$.

In summary, the learner's input includes: the spaces $\mathcal{X},\mathcal{Y},\mathcal{O}$, the hypothesis class $\+H$ and transition class $\+T$, the training set $S=\{x^{(i)},o^{(i)}\}_{i=1}^m$, and the loss functions $\ell$, $\ell_\mathcal{O}$.

A hypothesis class $\+H$ is said to be ($\+T$)-\emph{learnable} if there is a learning algorithm $\+A:\cup_{m=1}^\infty (\mathcal{X}\times\mathcal{O})^m\mapsto \mathcal H$ such that: for any distribution $D_X$ over $\+X$ and $T_0\in \+T$, when running $\+A$ on datasets $S^{(m)}$ of $m$ independent samples of $(X,O)$, we have $R(\+A(S^{(m)}))$ converges to 0 in probability as $m\rightarrow \infty$. 
In particular, we define the \emph{Empirical Risk Minimizer} to be a mapping $\operatorname{ERM}:\cup_{m=1}^\infty (\mathcal{X}\times\mathcal{O})^m\mapsto \mathcal H$ such that $\operatorname{ERM}(S)\in \argmin_{h\in \mathcal{H}, T\in \mathcal{T}}\widehat R_\mathcal{O}(h\circ T|S)$, where the $\operatorname{argmin}$ operator only returns the base hypothesis (although the empirical risk is minimized over both base and transition hypotheses). 
\section{General Learnability Conditions}

In this section, we present theorem \ref{learnability-conditions} that decomposes the learnability of a general indirect supervision problem into three aspects: \emph{complexity}, \emph{consistency} and \emph{identifiability}. After that, we provide proposition \ref{Natarajan2weakVC} to help verifying the complexity condition. The other two conditions will be further studied in the next section.

% \[
% \sigma \eqdef \sup_{h\in \mathcal H, T\in \+T} \sqrt{\E_{x,o}[\ell^2_\mathcal{O}(T,(x,h(x),o))]} \in (0,b]
% \]

We assume $\ell_\mathcal O$ takes value in an interval $[0,b]$ for some constant $b>0$.
% \qn{maybe explain this doesn't apply to cross-entropy and why we still think this is a reasonable assumption?}
To characterize the learnability, a key step is to describe the complexity of the function class
\[
\ell_\mathcal O\circ \mathcal{T} \circ \mathcal H \eqdef \left\{(x,o)\mapsto \ell_\mathcal{O}(T,(x,h(x),o)):h\in\mathcal H, T \in \mathcal{T}\right\}
\]
To do so, we use the following generalized version of VC-dimension proposed in \citep{Baraud16}. It will be shown in proposition \ref{Natarajan2weakVC} that it enables us to bound the Rademacher complexity \citep{BartlettMe03} of $\ell_\+O \circ \+T \circ \+H$ (which provides the flexibility to study arbitrary loss function) via the standard VC dimension or Natarajan dimension \citep{Nataraj89b} (also see Chapter 29 of \citep{Ben-DavidSh14} for an introduction) of $\+H$ (which is in general easier to compute than the Rademacher complexity).

\begin{definition} We adopt the following definitions from \citep{Baraud16}:
    \begin{enumerate} [leftmargin=*]
    \item \textnormal{(shattering and VC-class)}  
    A class $\mathcal{C}$ of subsets of a set $\mathcal{Z}$ is said to \emph{shatter} a finite subset $Z\subseteq\mathcal{Z}$ if
    \[
    \left\{C\cap Z:C\in\mathcal{C}\right\} = 2^Z
    \]
    Moreover, $\+C$ is called a \emph{VC-class} with dimension no larger than $k$ if there exists an integer $k$ such that $\+C$ cannot shatter any subset of $\mathcal{Z}$ with more than $k$ elements.
    \item \textnormal{(weak VC-major)}
    The function class $\ell_\mathcal O \circ \+T \circ \mathcal H$ is said to be \emph{weak VC-major} with dimension $d$ if $d$ is the smallest integer such that for all $u\in \mathbb R$, the set family
    \[
    \mathcal C_u\eqdef\left\{\{ (x,o):\ell_\mathcal O(h(x),T,(x,o))>u\}:h\in\mathcal H, T\in\+T\right\}
    \]
    is a VC-class of $\mathcal X\times\mathcal O$ with dimension no larger than $d$.
    \end{enumerate}
\end{definition}

Now we are able to state the main result of  this section:

\begin{theorem}\label{learnability-conditions}
    If the following conditions are satisfied
    \begin{enumerate}[label=\normalfont{[C\arabic*]}, ref=\textbf{[C\arabic*]}]
        \item\label{Complexity} \textnormal{(Complexity)}
        $\ell_\mathcal O \circ \mathcal{T} \circ \mathcal H$ is weak VC-major with dimension $d<\infty$.
        \item \label{Consistency} \textnormal{(Consistency)}
        $h_0 \in \underset{h \in \mathcal H, T\in\mathcal{T}}{\argmin}
        \footnote{This $\operatorname{argmin}$ operator only returns the base hypothesis.}
        R_\mathcal O(T\circ h)$.
        \item\label{Identifiability} \textnormal{(Identifiability)}  
        $\eta \eqdef \underset{h \in \mathcal H, T\in \+T: R(h)>0}{\inf} \dfrac{R_\mathcal O(T\circ h)-\inf_{T\in\+T}R_\mathcal O(T\circ h_0)}{R(h)}>0$.
    \end{enumerate}
    Then, $\+H$ is $\+T$-learnable. That is, for any $\delta\in(0,1)$, with probability of at least $1-\delta$, we have:
    \begin{equation} \label{learning-bound}
        \begin{aligned}
            R(\operatorname{ERM}(S^{(m)})) 
            \le \frac {2b}\eta \left( 
                 \sqrt{\frac{2\overline{\Gamma}_m(d)}{m}}+\frac{4\overline{\Gamma}_m(d)}{m} 
                 + \sqrt{\frac{2\log(4/\delta)}{m}} \right)
        \end{aligned}
    \end{equation}
    where $\overline{\Gamma}_m(d)$ is defined in \citep{Baraud16} by 
    \[
    \begin{aligned} \label{gamma-bar}
         \overline{\Gamma}_m(d) \eqdef \log \left[2\sum_{j=0}^{\min\{d,m\}}\binom{m}{j}\right] = d\log m(1+o(1)) \,\, \text{as} \,\, m\rightarrow \infty
    \end{aligned}
    \]
    % where $d\land m=\min\{d,m\}$. 
    This implies $R(\operatorname{ERM}(S^{(m)}))\rightarrow 0$ in probability as $m\rightarrow \infty$.
\end{theorem}

% \kaifucomment{optional: add the visualization.}
% We can visualize the conditions as follows.
% \begin{figure}[H]
%     \centering
%     \include{images/risk_curve}
%     \caption{title}
% \end{figure}

Bound \eqref{learning-bound} suggests that the difficulty of the learning can be characterized by (i) the identifiability level $\eta$, which mainly depends on the nature of the indirect supervision signal and the learner's prior information of the transition hypothesis, which will be further studied in the next section. (ii) the weak VC-major $d$ of $\ell_\mathcal O \circ T \circ \mathcal H$, which depends on the modeling choice. We present the following results that bound $d$ by the Natarajan dimension of $\mathcal{H}$ and the weak-VC major dimension of the class:
\[
\ell_\mathcal O\circ \mathcal{T} \eqdef \left\{(x,\widehat y, o)\mapsto \ell_\mathcal{O}(\widehat y,T,(x,o)): T \in \mathcal{T}\right\}
\]

\begin{proposition} \label{Natarajan2weakVC}
    Suppose the Natarajan dimension of $\mathcal{H}$ is $d_\mathcal H<\infty$ and the weak-VC major dimension of $\ell_\+O \circ \+T$ is $d_\+T<\infty$.
    Then, the weak-VC major dimension of $\ell_\mathcal O \circ \+T \circ \mathcal H$, 
    $d$, 
    can be bounded by:
    \[
    d \le 2\left( (d_\mathcal{H}+d_\+T)\log(6(d_\mathcal{H}+d_\+T))+2d_\mathcal{H}\log c \right)
    \;
    \text{where}
    \;
    c=|\+Y|
    \]
\end{proposition}

The reason that we do \emph{not} study the complexity of $\+T$ separately is that the annotation loss may be independent of $T$ (i.e., $\ell_\mathcal{O}(\widehat y,T_1,(x,o))=\ell_\mathcal{O}(\widehat y,T_2,(x,o))$ for any $T_1,T_2\in \+T$). See proposition \ref{Robustness} for an example of such a loss. 

To show applications of proposition \ref{Natarajan2weakVC}, we study the following cases:

\begin{example} \label{vc-example}
In the following cases, we first compute/bound $d_\+T$, then $d$ can be bounded by $d_\+H$:
% We have \qn{the following results (about $d_{\+T}$?)}: 

    \begin{enumerate} [leftmargin=*]
        \item When the true transition is known \emph{or} when the annotation loss function only depends on $(\widehat y, o)$, we have $d_\+T=0$; hence $d \le 2d_\mathcal{H}(\log(6d_\mathcal{H})+2\log c)$. This is conceptually similar to the Lemma 3.4 in \citep{LiuDi14}, which bounds the VC-dimension of the induced hypothesis class for the noise-free superset problem. 
        \item When all transition hypotheses in $\+T$ are instance-independent \emph{and} the annotation loss only depends on $(T,\widehat y, o)$ (e.g., the cross-entropy loss defined in \eqref{cross-entropy}), then $d_\+T$ can be trivially bounded by $d_\+T\le cs=|\+Y\times\+O|$; hence $d \le 2((d_\mathcal{H}+cs)\log(6(d_\mathcal{H}+cs))+2d_\mathcal{H}\log c)$.
        % \item \citep{} defines a generalized linear model
        \item Suppose the instance is embedded in a vector space $\+X=\mathbb{R}^p$. Consider the problem (Example 5.1.3 in \cite{MenonRoNa18}) of binary classification with a uniform noise rate which is modeled as a Logistic regression: $\P(O\ne y|x,y)=S(w^\T x)$ where $S$ is the sigmoid function and $w$ is the parameter. Then the cross-entropy loss becomes: $-\1\{o\ne\widehat{y}\}\log(S(w^\T x))-\1\{o=\widehat{y}\}\log(1-S(w^\T x))$. We have $d_\+T\le 2p+2$. See supplementary material for a proof.
    \end{enumerate}
\end{example}

% \begin{proposition}[Instance-independent Transition]
% \end{proposition}

% Theorem \ref{learnability-conditions} provides general conditions that guarantee learnability. To apply the theorem in real-world problems, it is necessary to develop additional tools to help verify if the conditions are satisfied. We present some useful results in the following lemmas.
\section{Separation}

Throughout this section we assume \ref{Complexity}
of Theorem \ref{learnability-conditions}
% \dr{of Thm.~\ref{learnability-conditions}; people will not know what you mean without good cross referencing; or you can just say at the beginning that you refer to the conditions of this theorem.}
holds. We will first propose a concept called \emph{separation}, which provides an intuitive way to understand the learnability and helps to verify \ref{Consistency} and \ref{Identifiability}; then we study two ways to ensure separation, and their application in real problems.

\subsection{Learning by Separation}

Without any prior knowledge, the transition class will contain all possible transitions.
% \qn{stochatic matrix is a special term: it's a square matrix with its row or column sum being 1; I don't think ours is a square matrix here, is it?}.
In this case, learnability cannot be ensured since a wrong label $\widehat{y}$ can also induce a good prediction of $O$ via an incorrect transition hypothesis. 
Hence, certain kind of prior knowledge is needed to restrict the range of $\mathcal{T}$. To formalize this idea, we first introduce an extension of the KL-divergence.

\begin{definition} [KL-divergence between Two Sets of Distributions]
    Given two sets of distributions $\mathcal{D}_1$ and $\mathcal{D}_2$, we define the \emph{KL-divergence} between them as:
    \[ 
    \KL{\mathcal{D}_1}{\mathcal{D}_2} \eqdef \inf_{D_1\in\mathcal{D}_1,D_2\in\mathcal{D}_2} \KL{D_1}{D_2}
    \]
\end{definition}
% Clearly, if $\mathcal{D}_1\cap\mathcal{D}_2\ne\emptyset$, then $\KL{\mathcal{D}_1}{\mathcal{D}_2}=0$. Also, if $\mathcal{D}_1,\mathcal{D}_2$ both contain one element, then it degenerates to the standard KL-divergence. 
Now we are able to state the main result of this section:

\begin{theorem} [Separation] \label{separation}
    % Let $\+D_\+O$ be the set of all the probabilistic distributions over $\+O$.
    For all $x\in\+X$, we denote the \emph{induced distribution families} by label $y_i$ as $\+D_i(x)\eqdef \{(T(x))_i:T\in\+T\}\subseteq \+D_\+O$ 
    (recall that $(T(x))_{i}$ is the $i^{\text{th}}$ row of $T(x)$),
    and the set of all possible predictions of the label as $\+H(x)\eqdef\{h(x):h\in\+H\}\subseteq\+Y$. Suppose
    \begin{equation} \label{separation-condition}
    \gamma \eqdef \inf_{(x,i,j):p(x,y_i)>0,j\ne i,y_j\in\+H(x)} \KL{\mathcal{D}_i(x)}{\mathcal{D}_j(x)} > 0
    \end{equation}
    Then $\mathcal{H}$ is learnable from the observations of $(X,O)$ with $\eta\ge \gamma > 0$ via the ERM of cross-entropy loss \eqref{cross-entropy}. We call $\gamma$ the \emph{separation degree}. 
    
    Moreover, if \eqref{separation-condition} is not satisfied, then there exists a sequence of transitions $\{T^{(k)}\}_k\,(T^{(k)}\in\+T)$ and distributions $\{D^{(k)}_X\}_k$ over $\+X$ such that $\lim_k\eta^{(k)}=0$ 
    % \qn{but $\eta$ can still be positive: is that why this is not strictly a necessary condition?}
    , where $\eta^{(k)}$ is defined the same as $\eta$ in \ref{Identifiability}, with the expectation
    % \qn{which expectation?}
    (in the definition of the risk functions) being taken according to $T^{(k)}$ and $D^{(k)}_X$.
\end{theorem}

\begin{figure}[h!]
    \centering
    \subfigure[]{\includegraphics[width=0.255\textwidth]{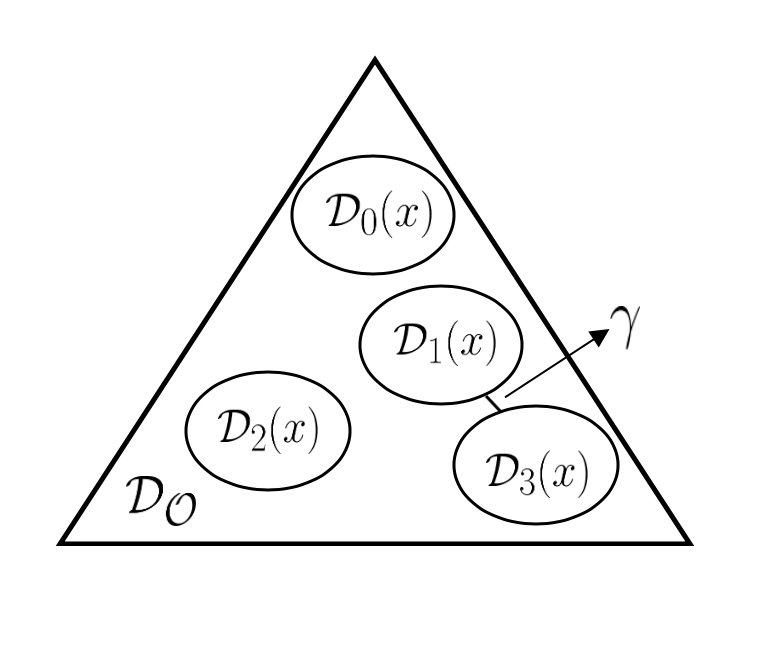}}
    \subfigure[]{\includegraphics[width=0.739\textwidth]{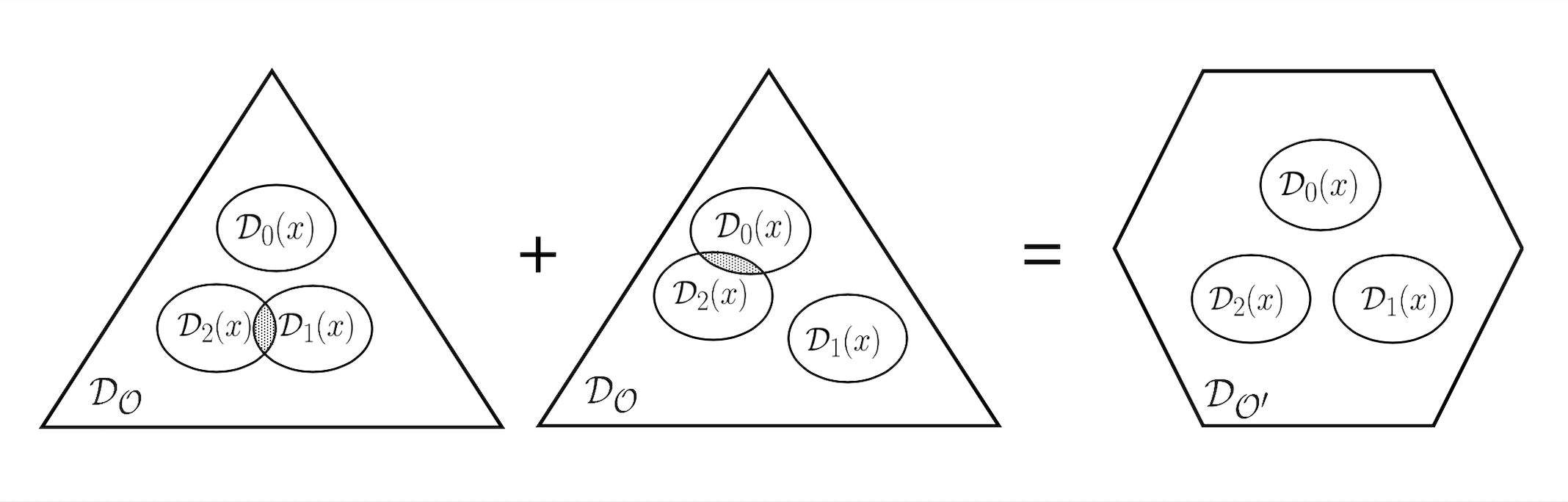}}
    % \subfigure[]{\includegraphics[width=0.28\textwidth]{images/average}}
    \caption{
    % \qn{in (a), $\gamma$ resides in a very narrow gap; how about putting it somewhere and adding an arrow pointing to that gap?}
        \textbf{(a)} Illustration of \emph{separation}. A (predicted) label $y_i$ will induce a distribution family over $\+O$ called $\mathcal{D}_i(x)$. Different families are separated by a minimal ``distance'' $\gamma$. 
        \textbf{(b)} Illustration of \emph{joint supervision}. By adding new supervision signals, separation of particular pairs of labels is preserved.
    }
    \label{fig:foobar}
\end{figure}

Theorem \ref{separation} is important in two ways: 
(i) It provides a way to characterize the prior knowledge of the learner about the transition using the KL-divergence and reveals its connection with the identifiability of labels. 
(ii) The ``moreove'' 
% \qn{this reference may be confusing to readers} 
result shows if separation is not satisfied, then the induced distribution of $O$ by different labels can be arbitrarily close,
% \qn{hence is an adv. so you keep missing an ``and'' when you use ``hence''} 
and hence the learning of $Y$ from $O$ can be arbitrarily difficult.
An illustration of separation is shown in \cref{fig:foobar} (a).
Yet, a drawback of the cross-entropy loss used in theorem \ref{separation} is that it can be unbounded when there is a zero element in the transition matrix.
% \qn{appear in the transition is unclear: a zero-element in $T$? a zero-element in $T\circ h$?}. \qn{partly solved? how about ``we will come back to this problem in ...''}
This problem will be partly solved in proposition \ref{Robustness} by introducing a different annotation loss. 

As the simplest application, we introduce the case where the transition is fully known to the learner. In this case, the induced distribution families $\+D_i(x)$ reduce to some points in the probability simplex.

\begin{example}[Full Information of the Transition]\label{Full-info}
    Suppose the transition $T(x)$ is known to the learner, i.e., $\+T=\{T_0\}$, by Theorem \ref{separation}, we know $\+H$ is learnable if 
    \[
    \inf_{(x,i,j):p(x)>0,\,i\ne j,y_j\in\+H(x)} \KL{(T_0(x))_i}{(T_0(x))_j} >0
    \]
    % In particular, this condition can be ensured if for all $x\in\+X$, no pairs of rows in $T_0(x)$ are colinear. 
    % \qn{cite which works require invertibility}
    Notice that this is a weaker assumption than the invertibility assumption of $T_0(x)$, which is used in \cite{RooyenWi18} (called \emph{reconstructible corruption}). This is possible because we assume a deterministic rule for $X\rightarrow Y$ but a randomized process for $X,Y\rightarrow O$, hence the latter one could contain more information and is capable of encoding
    the deterministic rule even when $\+O$ is smaller than $\+Y$. 
    For a concrete example, consider the following constant transition matrix:
    \[
    T = 
    \begin{bmatrix}
        0.1 & 0.9 \\ 
        0.5 & 0.5 \\ 
        0.9 & 0.1
    \end{bmatrix}
    \]
    In this case, $|\+Y|=3>|\+O|=2$ and therefore $T$ is not right-invertible. However, since the distribution of $O$ induced by each labels $Y$ is known to learner, and this distribution could be estimated from the observation of $O$, the learner is able to recover the true label from the indirect supervision.
    % \dr{Something wasn't clear to me in the above sentence and I am suggesting: However, the distribution of $O$ induced by the different $Y$ labels is known to learner this could be used to recover the true label.}
    % \kaifucomment{I have combined the two sentences. Please take a look.}
\end{example}

\subsection{Separation by Total Variation} \label{total-variation}

In this subsection, we introduce a way to guarantee separation by controlling the KL-divergence using total variation distance, which is done via the well-known Pinsker's inequality \citep{Tsybakov08}:

\begin{lemma} [Pinsker's inequality, proposed in \citep{pinsker}, see \citep{Tsybakov08} for an introduction]
    If $P$ and $Q$ are two probability distributions on the same measurable space $(\Omega,\mathcal{F})$, then
    \[
    \|P-Q\|_\text{TV} \le \sqrt{\KL{P}{Q}/2}
    \]
    where $\|\cdot\|_{\text{TV}}$ is the total variance distance: $\|P-Q\|_\text{TV} \eqdef \sup_{A\in\mathcal{F}} \|P(A)-Q(A)\|$.
    Moreover, if $\Omega$ is countable (in our case, $\Omega=\+O$ is finite and hence, then the total variance distance is equivalent to the $L^1$-distance in the sense that $\|P-Q\|_\text{TV} = \frac12\|P-Q\|_{1}$).
\end{lemma}

This lemma suggests we can ensure separation by controlling the $L^1$-distance. To show a concrete example, we introduce the \emph{concentration} condition. The intuition behind it is that the information of different labels in $\+Y$ is concentrated in relatively different sets of annotations. Formally: 

\begin{proposition}[Concentration] \label{Robustness}
    A sufficient condition for \eqref{separation-condition} is that for every $1\le i\le c$, there exists a set $S_i\subset \+O$ (we call them \emph{concentration sets}) such that
    \begin{equation} \label{Massart-condition}
    \gamma_C \eqdef \inf_{(i,j,x,T):T\in\+T, p(x)>0, j\ne i} \P_T(O\in S_i|x,y_i) -  \P_T(O\in S_j|x,y_i) > 0
    \end{equation}
    where $\P_T(\cdot)$ is the conditional probability defined by transition $T$. Under this condition, we can relate identifiability and separation degree by $\eta\ge \gamma\ge 2\gamma_C^2$. Since a condition imposed on all $T\in\+T$ can be regarded as an assumption imposed on the true transition $T_0$, condition \eqref{Massart-condition} can be rewritten as: 
    \begin{equation} \label{concentration-condition}
        \gamma_C=\inf_{(i,j,x):p(x,y_i)>0, i\ne j} \P(O\in S_i|x,y_i) -  \P(O\in S_j|x,y_i) > 0
    \end{equation}
    Also, in this case, one can ensure learnability by the ERM which minimizes the following transition-independent annotation loss 
    \begin{equation} \label{generalized-zero-one}
        \ell_\mathcal{O}(h(x),T,(x,o)) \eqdef \1\{o\notin S_{h(x)}\} 
    \end{equation}
    For this annotation loss, we can bound the identifiability level by $\eta\ge \gamma_C$.
    % \qn{a bit confused about what this proposition is. What's the final claim then?}
\end{proposition}

% \begin{proof}
    
% \end{proof}

\begin{example}[Superset with Noise] \label{superset-example}
    For superset with noise problem where $O$ is a random subset of $\+Y$ (i.e., $\+O=2^\+Y$), let $S_i=\{o:y_i\in o\}\subset \+O$, the conditional \eqref{concentration-condition} becomes
    \begin{equation} \label{superset-conditon}
        \gamma_C=\inf_{p(x,y_i)>0, i\ne j} \P(y_i\in O|x, y_i) - \P(y_j\in O|x, y_i) > 0  
    \end{equation}
    This generalizes the \emph{small ambiguity degree condition} proposed in \citep{CourSaTa11, LiuDi14}, which assumes $ \P(y_i\in O|x, y_i)=1$ (i.e., the gold label always lies in the superset). \citep{CourSaTa11, LiuDi14} also proposes a \emph{superset loss}, which is the special case of \eqref{generalized-zero-one}.
    We extend the discussion to allow the presence of noise. 
    % \qn{how about ``The small ambiguity condition \citep{superset-Taskar,superset-theory} assumes ..., but here our framework can generalize it to allow the presence of noise, which is a new result in the literature.''}
\end{example}

The following example can be regarded as a special case of Example \ref{superset-example}. 
% when each superset $O$ contains exactly one element.

\begin{example}[Label Noise] \label{noise}
    For noisy problem where $\+O=\+Y$, let $S_i=\{y_i\}$, condition \eqref{Massart-condition} becomes
    \begin{equation} \label{noisy-condition}
        \gamma_C=\inf_{p(x,y_i)>0, i\ne j} \P(O=y_i|x, y_i) - \P(O=y_j|x, y_i) > 0  
    \end{equation}
    This generalizes the Massart noise condition \citep{MassartNe06} of binary classification, which assumes the noise rate is lower bounded by $1/2$ minus a constant.
    % \qn{``lower bounded away from'' sounds weird. Is there a better way to say it?}. 
    We extend the discussion to multiclass case.

    Also, notice that \eqref{generalized-zero-one} is simply the zero-one loss for $O$, which means learnability can still be guaranteed if one ignores the noisy process and learns $O$ as clean label. This partly explains the empirical study in \citep{RVBS17}, which tests the robustness of neural networks (without additional denoising process) to noise in annotations. \citep{RVBS17} proposes a parameter 
    % \qn{how about ``proposes an empirical metric''}
    called $\delta$-degree which is similar to $\gamma_C$ and observes that the performance of the network decreases as $\delta$ decreases, as our learning bound \eqref{learning-bound} suggests.
    % Although the noisy problem is relatively well-studied in literature, our discussion provides new perspectives to view this problem: (i)
\end{example}

We can further generalize proposition \ref{Robustness} by encoding \emph{functional} prior information of the transition: 

\begin{proposition}[Evidence] \label{Evidence} 
    A sufficient condition for \eqref{separation-condition} is that there exists Lipschitz (with respect to the $L^1$-norm of vectors) functions $\Phi_{ij}:\mathbb{R}^c\rightarrow \mathbb{R}, 1\le i\ne j\le c$ (we call them \emph{evidence}) with Lipschitz constants $L_{ij}$ such that
    \begin{equation} \label{evidence-condition}
        \gamma_{ij} \eqdef
        \inf_{p(x,y_i)>0, y_j\in\+H(x), D_i\in \+D_i(x), D_j\in \+D_j(x)} 
        \Phi_{ij}\left(D_i\right) - \Phi_{ij}\left(D_j\right)>0
    \end{equation}
    % \begin{equation} \label{evidence-condition}
    %     \gamma_{ij} \eqdef
    %     \inf_{T\in\+T, p(x,y_i)>0, y_j\in\+H(x)} \Phi_{ij}\left((T(x))_i\right) 
    %     -
    %     \sup_{T\in\+T, p(x,y_i)>0, y_j\in\+H(x)}\Phi_{ij}\left((T(x))_j\right)>0
    % \end{equation}
    In this case, the separation degree can be bounded by $\gamma\ge 1/2 \min_{i\ne j} \left({\gamma_{ij}/L_{ij}}\right)^2$.

    In particular, the dot product with a fixed vector $\Phi_u(t)=\langle u,t\rangle$ ($\langle\cdot,\cdot\rangle$ represents the dot product) is Lipschitz with Lipschitz constant $L_u\le \|u\|_\infty$. 
    As an example, given sets $S_i\subset \+O\;(1\le i\le c)$, letting 
    $\Phi_{ij}(\cdot)=\langle\sum_{k:o_k\in S_i} \hat e_k -\sum_{k:o_k\in S_j} \hat e_k ,\cdot\rangle$ 
    recovers proposition \ref{Robustness}, where $\hat e_k$ is the $k^\text{th}$ standard unit basis vector of $\mathbb{R}^c$. Another example will be given in Example \ref{better-job}.
\end{proposition}

\subsection{Separation by Joint Supervision} \label{joint-supervision}

When a weak supervision signal cannot ensure learnability individually, it needs to be used with other forms of annotations together to supervise the learning. 
Our goal in this subsection is to provide a way to describe the effect of using multiple sources of annotations jointly. 
We will show that joint supervision can improve (Example \ref{better-job}), preserve (Proposition \ref{no-free-separation}) or even damage (Remark \ref{defining-intersection}) the separation.
% \qn{where do you explain the ``damage'' part?}

First, we formulate the joint supervision problem. For simplicity, we only consider the case that we have two sources of annotations $O_1,O_2$, and the general case can be discussed in a similar way. For each $O_k, k\in\{1,2\}$, denote its annotation space as $\mathcal{O}_k$, its transition as $T_k(x)$ and its transition classes as $\+T_k$. We focus on the scenario that for each instance $x$, there is only \emph{one} type of annotation.
Then the joint annotation space is $\+O=\+O_1\cup\+O_2$.
We model the \emph{annotation type} $\1\{O=O_k\}$ as a random variable that is independent with $X$ and all the $O_k$, and the probability $\P(O=O_1)=\lambda$ is known to the learner.
Then the \emph{joint annotation} is defined as: $O = \1\{O=O_1\}O_1+\1\{O=O_2\}O_2$. 

Next, we quantify the supervision power of an annotation if separation is not guaranteed via a local version of the separation (degree):

\begin{definition}[Pairwise Separation] 
    Define the \emph{separation degree of $y_i$ to $y_j$} as
    % We say the labels $y_i$ is separated from a $y_j\ne y_i$ if 
    \begin{equation} \label{label-specific-separation}
        \gamma_{i\rightarrow j} \eqdef \inf_{x:p(x,y_i)>0,y_j\in\+H(x)} \KL{\mathcal{D}_i(x)}{\mathcal{D}_j(x)} 
        % > 0
    \end{equation}
    We say the labels $y_i$ is \emph{separated from} a $y_j$ if $\gamma_{i\rightarrow j}>0$. The separation degree $\gamma = \min_{i,j}\gamma_{i\rightarrow j}$.
    % Moreover, define $R(h|Y=y_i) = \E_{X,Y}[\ell(h(X),Y)\1\{Y=y_i\}]$. Suppose $\gamma_i = \min_{j\ne i}\gamma_{ij}>0$, we have
    % \begin{equation} \label{label-specific-bound}
    %     \begin{aligned}
    %     R(\operatorname{ERM}(S^{(m)})|Y=y_i) 
    %     \le \frac {4b}{\gamma_i} \left( 
    %          \sqrt{\frac{2\overline{\Gamma}_m(d)}{m}}+\frac{4\overline{\Gamma}_m(d)}{m} 
    %          + \sqrt{\frac{\log(4/\delta)}{2m}} \right)
    %     \end{aligned}
    % \end{equation}
\end{definition}

This definition gives a probabilistic formulation of the intuition that a (weak) supervision signal can help \emph{distinguish} certain pairs of labels.
For example, a noisy annotation for multiclass classification may break the condition \eqref{noisy-condition} due to a large noise rate for certain labels, but it can still provide information to separate other labels if \eqref{noisy-condition} is satisfied for any other pairs of $(i,j)$.
% Another example is the structured prediction problem discussed in \citep{binary-indirect-supervision}, an auxiliary binary annotation is provided to indicate if the label is ``$\text{Negative}$'' or not. This annotation can provide separation of ``$\text{Negative}$'' from all other labels.
% Another example is the \emph{semantic parsing} problem discussed in \citep{world-response}, a sentence $x$ is parsed into a logical form $y$, which is used to interact with a database, and a binary feedback signal. 

When there are no additional constraints on the joint transition, one can construct the joint transition simply by combining the candidate transitions in $\+T_1,\+T_2$. For example, the induced distribution family by $y_i$ of joint supervision can be naturally constructed by 
% \qn{a bit lost here}
\begin{equation} \label{no-constraint-joint-transition}
    \+D_i(x) = \{\lambda D_1 + (1-\lambda) D_2: D_1\in \+D_{i1}(x), D_2\in \+D_{i2}(x)\}
\end{equation}
where $\+D_{i1}$ and $\+D_{i2}$ are the induced distribution family by $y_i$ of $O_1$ and $O_2$.
In this case, we present the following result to characterize the learnability under joint supervision $O$:

\begin{proposition}[No Free Separation] \label{no-free-separation}

Suppose the separation degrees of $y_i$ to $y_j$ of $O_1$ and $O_2$ are $\gamma_{i\rightarrow j1}$ and $\gamma_{i\rightarrow j2}$ respectively.
% \qn{how about $\gamma_{i\rightarrow j}^{(1)}$}. 
Then, if the joint transition class is constructed as \eqref{no-constraint-joint-transition}, then the separation degrees of $y_i$ to $y_j$ for the joint supervision satisfies:
    \[
        \gamma_{i\rightarrow j} \le 
        \lambda \gamma_{i\rightarrow j1} + 
        (1-\lambda) \gamma_{i\rightarrow j2}
    \]
Also, if $\+O_1\cap \+O_2=\emptyset$, then the two equality holds. As a consequence, a necessary condition of that $y_i$ is separated from $y_j$ by the joint signal $O$ is that $y_i$ must be separated from $y_j$ by one of $O_1,O_2$.
\end{proposition}

\begin{remark}[Defining $\+O_1\cap \+O_2$] \label{defining-intersection}
The condition $\+O_1\cap \+O_2=\emptyset$ means that the learner \emph{distinguishes} different annotations. 
For example, in a crowdsourcing setting, we have two annotators and each provides a noisy annotation, then $\+O_1,\+O_2=\+Y$. But as long as the learner distinguishes the annotations of the two annotators,
% \qn{refer to the ``Are We Modeling the Task or the Annotator?'' paper; I think this explains their observation.}
we can nevertheless write $\+O_1\cap \+O_2=\emptyset$. Without this condition, even if both $\gamma_{i\rightarrow j1},\gamma_{i\rightarrow j2}>0$ , we can still have $\gamma_{i\rightarrow j}=0$. 
See the supplementary material for an example.
This idea has also been explored in the empirical study of \citep{GevaGoBe19}, which observes that in a crowdsourcing setting, the model performance improves if annotator identifiers are input as features. 
However, one should note that the tradeoff is the model complexity: distinguishing different annotations will in general require more parameters to model the joint transition.
\end{remark}
% $T_{\text{joint}}(x) = 
% \begin{bmatrix}
%         \lambda_1 T_1(x) & \lambda_2 T_2(x) & \cdots & \lambda_n T_n(x)
% \end{bmatrix}$

\begin{remark}
Proposition \ref{no-free-separation} shows that without constraints, the joint supervision does not create new separation, however, it can preserve the separation between labels by the original supervision signals. 
So in this view, the weak supervision signal can be regarded as a ``building block'' for the (global) separation \eqref{separation-condition} by contributing pairwise separation \eqref{label-specific-separation}.
% When a label can be separated from all the other labels, the supervision power can be quantified by the label-specific bound \eqref{label-specific-bound}.
An illustration is shown in \cref{fig:foobar} (b). 
\end{remark}

If there does exist constraints about the two transition classes, Proposition \ref{no-free-separation} no longer holds and joint supervision may create new separation. To illustrate, consider the following artificial example:

\begin{example}[Learning from Difference] \label{better-job}
Given a binary classification problem where $\+Y=\{\pm 1\}$. Suppose we have two annotators $O_1$ and $O_2$ and each provides a noisy annotation with an unknown, uniform, instance-independent noise, i.e., 
$\eta_1\eqdef\P(O_1\ne y|x,y=-1)=\P(O_1\ne y|x,y=+1)$, 
$\eta_2\eqdef\P(O_2\ne y|x,y=-1)=\P(O_2\ne y|x,y=+1)$, 
where $\eta_1,\eta_2$ are constants independent of $x$. Then, the joint transition is modeled as:
\[
    T=
\begin{bmatrix}
    \lambda (1-\eta_1) & \lambda \eta_1 & (1-\lambda) (1-\eta_2) & (1-\lambda)\eta_2 \\ 
    \lambda \eta_1 & \lambda (1-\eta_1)  & (1-\lambda)\eta_2 & (1-\lambda) (1-\eta_2) \\ 
\end{bmatrix}
=
\begin{bmatrix}
    D_1\\ 
    D_2\\ 
\end{bmatrix}
\]
Now, suppose it is known that the first annotator provides a better quality of annotation, i.e., there is a $\gamma\in\mathbb{R}$ (known to the learner) such that $\eta_1-\eta_2\le \gamma<0$. 
To apply proposition \ref{Evidence}, define the evidence  
$\Phi(\cdot)=\langle \hat e_1/\lambda - \hat e_3/(1-\lambda),\cdot\rangle$,
then for any $D_1, D_2$, we have $\Phi(D_1)=\eta_2-\eta_1\ge \gamma$ and  $\Phi(D_2)=\eta_1-\eta_2\le-\gamma$. So by proposition \ref{Evidence}, the original classification hypothesis is learnable. Notice that without joint supervision, separation is not guaranteed since we have no restriction on $\eta_1$ or $\eta_2$ individually.

This example shows the necessity to model possible constraints between different supervision sources, which help to reduce the size of the joint transition class and may improve the separation degree.
\end{example}
\section{Conclusion and Future Work}

In this paper, we provide a unified framework for analyzing the learnability of multiclass classification with indirect supervision.
Our theory builds upon two key components:
(i) The construction of the induced hypothesis class and its complexity analysis, which allows us to indirectly supervise the learning by minimizing the annotation risk.
(ii) % \dr{I don't understand the next sentence; something is wrong} 
% \kaifucomment{it is based on three things: (i) transition class and separation describes prior knowledge (ii) separation bounds identifibility (iii) identifibility is the ratio of classification error and annotation risk. Which one do you think is wrong?}
% \dr{It's the phrasing; how about: ``A formal description of the prior knowledge about the transition and its encoding in the learning condition and bound; this allows us to bound the classification error by the annotation risk."}
A formal description of the learner's prior knowledge about the transition and its encoding in the learning condition, which allows us to bound the classification error by the annotation risk.
% \drc{\st{A formally description the prior knowledge about the transition, and encode it in the learning condition and learning bound, which allows us to bound the classification error by the annotation risk.}} This is done by defining \drc{\st the}} \emph{transition class} and \emph{separation}.

The notion of separation depends on the annotation loss being used. The KL-divergence may be replaced by other statistical distances, as long as the distance can induce a loss function. However, the idea behind separation is invariant: the prior knowledge needs to be strong enough to distinguish different labels via the observables.
Moreover, theorem \ref{separation} shows that separation is a sufficient and almost necessary condition, and the later examples show separation is also practically useful and can easily produce learnability conditions. 
Therefore, we believe the the concepts introduced are general, and that our analysis tools can be applied in many other supervision scenarios.

One limitation of our work is that the definition of learnability requires us to handle every possible $\+D_X$, and the consequence is that we need to ensure separation at every $x\in\+X$. In future work, we may try to relax the learnability conditions by encoding prior knowledge of $\+D_X$, which can be obtained from unlabeled datasets.
Another direction to explore is to extend the discussion to the agnostic case as well as the case where $T_0\notin \+T$. Also, in a directly supervised learning setting, classical realizable PAC learning could achieve a convergence rate of $O(\log(1/\delta)/m)$, which is better than the rate of $O(\sqrt{\log(1/\delta)/m})$ we derived for a general indirect supervision problem. It is worth exploring whether and how our bounds could be improved for more kinds of supervision signals (other than the gold label).

% (Global supervision) 

% (Prior knowledge of $\+D_\+X$)

% ****************************eight pages of content end here****************************

\section*{Broader Impact}

Our work mostly focuses on theoretical aspects of learning, however, it provides better understanding and thus can suggest new machine learning scenarios and algorithms for learning from indirect observations; this addresses a key challenge to machine learning today, and will help machine learning researchers to reduce the cost of and need for labeled data. 
Our theory may have positive and negative impact on the privacy protection of sensitive data. 
On one hand, the theory suggests that one can alter the forms of data (via a probabilistic transition) to ensure privacy while keeping its usefulness (learnability).
On the other hand, it might be possible for an attacker to recover sensitive information about the data indirectly through a related dataset.

\begin{ack}
  This work was supported by the Army Research Office under Grant Number W911NF-20-1-0080 and by contract FA8750-19-2-0201 with the US Defense Advanced Research Projects Agency (DARPA). The views expressed are those of the authors and do not reflect the official policy or position of the Department of Defense or the U.S. Government.
\end{ack}

\section{Appendix}

% ------------------------------------------------------------------
% main thm
% ------------------------------------------------------------------

\subsection{Proof of Theorem \ref{learnability-conditions}}

We need several intermediate results to prove this. First, we introduce the definition of the averaged Rademacher complexity.

\begin{definition} [Averaged Rademacher Complexity \citep{BartlettMe03}]
    The averaged Rademacher complexity \citep{BartlettMe03} of $\ell_\mathcal O \circ \+T \circ \mathcal H$ with respect to $m$ samples is defined as
    \begin{equation} \label{Rademacher}
    \mathfrak{R}_m(\ell_\mathcal O \circ \+T \circ \mathcal H)
    \eqdef
    \mathbb{E}_{\epsilon,x,o}
    \left[\frac1m \sup_{h\in\mathcal H,h\in\+T} \left| \sum_{i=1}^m\epsilon^{(i)}\ell_\mathcal{O}(h(x^{(i)}),T,(x^{(i)},o^{(i)}))\right|\right] 
    \end{equation}
    where $\epsilon_i\stackrel{\text{iid}}{\sim}\operatorname{Uniform}\{-1,+1\}$ are the so-called Rademacher random variables and the expectation is taken over $m$ i.i.d. samples of $\epsilon,x,o$ . 
\end{definition}

The first lemma bounds the empirical risk via the averaged Rademacher complexity.

\begin{lemma}[Adapted from the proof of Theorem 26.5 in \citep{Ben-DavidSh14}] \label{Rademacher-bound}
    In this lemma and its proof, for convenience, we let the $\operatorname{ERM}$ algorithm return the induced hypothesis in $\+H\circ \+T$ (rather than the base hypothesis only). 
    % \qn{note $h\circ T$}

    Given any $\delta\in(0,1)$, with probability of at least $1-\delta$, we have
    \[
        R_\mathcal{O}(\operatorname{ERM}(S^{(m)}))-\inf_{h,T}R_\mathcal O(T\circ h)
        \le 
        2 \mathfrak R_m(\ell_\+O \circ \+T \circ \+H) 
        + 2b\sqrt{\frac{2\log(4/\delta)}{m}}
    \]
\end{lemma}

\begin{proof}
    Let $T^\star\circ h^\star$ be any induced hypothesis in $\+T\times \+H$.
    % Following the method used in the proof of theorem 26.5 in \citep{Shalev-Shwartz:2014:UML:2621980}. 
    Given dataset $S^{(m)}$, we have,
    \[
    \begin{aligned}
        & R_\mathcal{O}(\operatorname{ERM}(S^{(m)}))-R_\mathcal O(T^\star \circ h^\star)\\ 
        =\; & R_\mathcal{O}(\operatorname{ERM}(S^{(m)}))
            - \widehat R_\mathcal{O}(\operatorname{ERM}(S^{(m)}))
            + \underset{\le 0}{\underbrace{\widehat R_\mathcal{O}(\operatorname{ERM}(S^{(m)}))
            - \widehat R_\mathcal{O}(T^\star \circ h^\star)}}\\ 
          & + \widehat R_\mathcal{O}(T^\star \circ h^\star)
            - R_\mathcal O(T^\star \circ h^\star)\\ 
        \le\; &  R_\mathcal{O}(\operatorname{ERM}(S^{(m)}))
            - \widehat R_\mathcal{O}(\operatorname{ERM}(S^{(m)}))
            + \widehat R_\mathcal{O}(T^\star \circ h^\star)
            - R_\mathcal O(T^\star \circ h^\star) 
    \end{aligned}
    \]
    By Theorem 26.5 (i) of \citep{Ben-DavidSh14}, we have that with probability of at least $1-\delta/2$,
    \[ 
         R_\mathcal{O}(\operatorname{ERM}(S^{(m)}))
        - \widehat R_\mathcal{O}(\operatorname{ERM}(S^{(m)}))
        \le 
        2 \mathfrak{R}_m'(\ell_\mathcal O \circ \+T \circ \+H)
        + b\sqrt{\frac{2\log(4/\delta)}{m}}
    \]
    where $\mathfrak R_m'(\ell_\mathcal O \circ \+T \circ \mathcal H)$ is defined slightly differently in \citep{Ben-DavidSh14} as:
    \begin{equation}
        \mathfrak{R}_m'(\ell_\mathcal O \circ \+T \circ \mathcal H)
        \eqdef
        \mathbb{E}_{\epsilon,x,o}
        \left[\frac1m \sup_{h\in\mathcal H,h\in\+T}  \sum_{i=1}^m\epsilon^{(i)}\ell_\mathcal{O}(h(x^{(i)}),T,(x^{(i)},o^{(i)})) \right] 
    \end{equation}
    It can be seen that $\mathfrak{R}_m'(\ell_\mathcal O \circ \+T \circ \mathcal H)\le \mathfrak R_m(\ell_\mathcal O \circ \+T \circ \+H)$ since the two quantities only differ by the absolute value.  Hence
    \[ 
        R_\mathcal{O}(\operatorname{ERM}(S^{(m)}))
        - \widehat R_\mathcal{O}(\operatorname{ERM}(S^{(m)}))
        \le 
        2 \mathfrak R_m(\ell_\mathcal O \circ \+T \circ \+H)
        + b\sqrt{\frac{2\log(4/\delta)}{m}}
    \]
    Also, by Hoeffding’s inequality, we have that with probability of at least $1-\delta/2$, 
    \[ 
        \widehat R_\mathcal{O}(T^\star \circ h^\star)
            - R_\mathcal O(T^\star \circ h^\star) 
        \le
        b\sqrt{\frac{\log(4/\delta)}{2m}}
    \]
    Combining the inequalities, we have that with probability of at least $1-\delta$, 
    \[
    \begin{aligned}
        R_\mathcal{O}(\operatorname{ERM}(S^{(m)}))-
        R_\mathcal O(T^\star \circ h^\star)
        \le
        2 \mathfrak R_m (\ell_\mathcal O \circ \+T \circ \mathcal H) + 2b\sqrt{\frac{2\log(4/\delta)}{m}}
    \end{aligned}
    \]
    Since the above inequality holds for any $T^\star \circ h^\star\in \+T\times \+H$,  taking infimum for all $T^\star \circ h^\star$ gives the desired result. 
    % \qn{why? a smaller $R_\mathcal{O}(\operatorname{ERM}(S^{(m)}))-R_\mathcal O(h^\star\circ T^\star)\le R_\mathcal{O}(\operatorname{ERM}(S^{(m)}))-\inf{R_\mathcal O(h^\star\circ T^\star)}$}
\end{proof}

The second lemma bounds the averaged Rademacher complexity via the weak VC-major, which is provided in \citep{Baraud16}.

\begin{lemma}[Adapted from the Theorem 2.1 in \citep{Baraud16}] \label{vc-Rademacher}
    Suppose the weak VC-major dimension of $\ell_\mathcal O \circ \+T \circ \mathcal H$ is $d$. then,

    \begin{equation} \label{complexity-bound}
        m\mathfrak R_m(\ell_\mathcal O \circ \+T \circ \mathcal H)
        \le  
        \sigma\log\left(\frac{\mathrm{e}b}{\sigma}\right)\sqrt{2m\overline{\Gamma}_m(d)} + 4b\overline{\Gamma}_n(d)
    \end{equation}
    where $\mathrm{e}$ is the base of the natural logarithm and
    \begin{equation} \label{variance}
    \sigma \eqdef \sup_{h\in \mathcal H, T\in \+T} \sqrt{\E_{x,o}[\ell^2_\mathcal{O}(T,(x,h(x),o))]} \in (0,b]
    \end{equation}
\end{lemma}

\begin{proof}
    The proof of the Theorem 2.1 in \citep{Baraud16} is long and is presented in the section 3 of \citep{Baraud16}. Here we only point out how to use Theorem 2.1 of \citep{Baraud16} (equation (2.8) of the paper) to derive our lemma. 
    
    First, the Theorem 2.1 of \citep{Baraud16} bounds an empirical process (denoted as $\E[Z(\+F)]$ in the paper, where $\+F$ is a function class and here we let $\+F=\ell_\mathcal O \circ \+T \circ \mathcal H$) rather than the averaged Rademacher complexity (denoted as $\E[\overline Z(\+F)]$ in the paper). However, the proof of Theorem 2.1 of \citep{Baraud16} aims to bound the averaged Rademacher complexity $\E[\overline Z(\+F)]$ and then uses the relation $\E[Z(\+F)]\le 2 \E[\overline Z(\+F)]$ (Lemma 2.1 of \citep{Baraud16}) to obtain the bound for $\E[Z(\+F)]$. Therefore, the proof of the Theorem 2.1 in \citep{Baraud16} tells:
    \begin{equation} \label{empirical-process-bound}
        \E[\overline Z(\+F)] = m\mathfrak R_m(\ell_\mathcal O \circ \+T \circ \mathcal H)
        \le 
        \sigma\log\left(\frac{\mathrm{e}}{\sigma}\right)\sqrt{2m\overline{\Gamma}_m(d)} + 4\overline{\Gamma}_n(d)
    \end{equation}
    Second, in the Theorem 2.1 of \citep{Baraud16}, it is assumed that the functions in $\+F$ is bounded in the interval $[0,1]$. Hence, we need scale the annotation loss to $\ell_\+O/b$ in order to use the theorem (i.e., let $f=\ell_\+O/b$ in the definition of $\E[Z(\+F)]$, i.e., equation (1.2) of \citep{Baraud16}). Also, in this case, the supreme of variance \eqref{variance} is scaled to $\sigma/b$. So, the inequality \eqref{empirical-process-bound} is rewritten as:
    \begin{equation}
        \E[\overline Z(\+F)] = \frac{m}{b}\mathfrak R_m(\ell_\mathcal O \circ \+T \circ \mathcal H)
        \le 
        \frac{\sigma}{b}\log\left(\frac{\mathrm{e}}{\sigma/b}\right)\sqrt{2m\overline{\Gamma}_m(d)} + 4\overline{\Gamma}_n(d)
    \end{equation}
    Rearranging the inequality gives the desired result.
\end{proof}

Now, we are able to give the proof of the original theorem:

\begin{proof} 
    % \qn{[Main theorem]}
    By consistency \ref{Consistency}, we have 
    \[ 
        \inf_{h,T}R_\mathcal O(T\circ h) = \inf_T R_\mathcal O(T\circ h_0)
    \]
    Therefore, by lemma \ref{Rademacher-bound}, we have that with probability of at least $1-\delta$,
    \[
        R_\mathcal{O}(\operatorname{ERM}(S^{(m)}))-\inf_T R_\mathcal O(T\circ h_0)
        \le 
        2 \mathfrak R_m(\ell_\mathcal O \circ \+T \circ \mathcal H) 
        + 2b\sqrt{\frac{2\log(4/\delta)}{m}}
    \]
    By identifiability \ref{Identifiability}, we have that with probability of at least $1-\delta$,
    \begin{equation} \label{common-bound}
    R(\operatorname{ERM}(S^{(m)})) \le \frac1\eta\left(2 \mathfrak R(\ell_\mathcal O \circ \+T \circ \mathcal H) + 2b\sqrt{\frac{2\log(4/\delta)}{m}}\right)
    \end{equation}

    By \ref{Complexity} and lemma \ref{vc-Rademacher}, we bound the Rademacher Complexity by
    \begin{equation} \label{complexity-bound}
        \begin{aligned}
        \mathfrak R_m(\ell_\mathcal O \circ \+T \circ \mathcal H)
        & \le  
        \sigma\log\left(\frac{\mathrm{e}b}{\sigma}\right)\sqrt{\frac{2\overline{\Gamma}_m(d)}{m}} + 4\frac{b}{m}\overline{\Gamma}_n(d) \\ 
        & \le  
        b\sqrt{\frac{2\overline{\Gamma}_m(d)}{m}} + 4\frac{b}{m}\overline{\Gamma}_n(d)
        \end{aligned}
    \end{equation}
    Now the result follows by combining \eqref{common-bound} and \eqref{complexity-bound}.
\end{proof}

% ------------------------------------------------------------------
% weak-VC
% ------------------------------------------------------------------

\subsection{Proof of Proposition \ref{Natarajan2weakVC}}

\begin{proof}
    First, we translate weak-VC major to the language of standard VC-dimension \citep{VapnikCh71}: 
    For a fixed $u\in\mathbb{R}$ and every $h\in\mathcal{H},T\in\+T$, we define an binary classifier: $f_{h,T,u}(x,o)=\mathbb{1}\left\{\ell_\mathcal O(h(x),T,(x,o))>u\right\}$ and denote $\mathcal{F}_u:=\{f_{h,T,u}:h\in\mathcal{H},T\in\+T\}$ as the set of such classifiers. Then $\mathcal{C}_u$ \emph{shatters} a set in $\mathcal{X}\times\mathcal{O}$ if and only if $\mathcal{F}_u$ shatters (in VC theory) the same set, so $\ell_\mathcal O \circ \+T \circ \mathcal H$ is weak VC-major with dimension $d$ if $d=\max_{u\in\mathbb{R}}\operatorname{VC}(\mathcal{F}_u)<\infty$, where $\operatorname{VC}(\cdot)$ is the VC dimension for hypothesis class of binary classifiers.
    
    Let $M$ be the maximum number of distinct ways to classify $d$ points in $\+X$ by $\+H$. Then for $d$ points in $\+X\times\+O$, suppose there are at most $M$ ways to assign multi-class labels $Y$ to each point. By Natarajan’s lemma \citep{Nataraj89b} of multiclass classification, we have
    \begin{equation} \label{Natarajan-count}
        M \le d^{d_\mathcal{H}}c^{2d_\mathcal{H}}
    \end{equation}
    For each way of assignment, it produces a set of $d$ points in $\+X\times\+Y\times\+O$, and for these $d$ points, by Sauer-Shelah lemma, there are at most 
    \[\sum_{i=0}^{d_\+T} \binom{d}{i} \le \left(\frac{\text{e}d}{d_\+T}\right)^{d_\+T}
    \] 
    ways to classify if $\ell_\+O(\widehat{y},T,(x,o))>u$ by $\+T$, so in total we have
    \[
    2^d \le M \sum_{i=0}^{d_\+T} \binom{d}{i} \le M\left(\frac{\text{e}d}{d_\+T}\right)^{d_\+T}
    \]
    where $\text{e}$ is the base of the natural logarithm. Therefore, $M\ge 2^d\left({d_\+T}/{\text{e}d}\right)^{d_\+T}$.
    Then, by \eqref{Natarajan-count} 
    \[
    d^{d_\mathcal{H}}c^{2d_\mathcal{H}} \ge M \ge 2^d\left(\frac{d_\+T}{\text{e}d}\right)^{d_\+T}
    \]
    Taking logarithm in both side, we have
    \[
    d_\mathcal{H}\log d+2d_\mathcal{H}\log c \ge d\log2 + d_\+T(\log d_\+T - \log d -1)
    \]
    Rearrange the inequality,
    \[
    \begin{aligned}
        d\log2 + d_\+T(\log d_\+T -1) & \le (d_\mathcal{H}+d_\+T)\log d+2d_\mathcal{H}\log c \\
        & \le (d_\mathcal{H}+d_\+T)\left(\frac{d}{6(d_\mathcal{H}+d_\+T)}+\log(6(d_\mathcal{H}+d_\+T))-1\right)+2d_\mathcal{H}\log c \\
        & = d/6 + (d_\mathcal{H}+d_\+T)\left(\log(6(d_\mathcal{H}+d_\+T))-1\right)+2d_\mathcal{H}\log c \\
        & \le d/6 + (d_\mathcal{H}+d_\+T)\log(6(d_\mathcal{H}+d_\+T))+2d_\mathcal{H}\log c 
    \end{aligned}
    \]
    where the second step follows from the first-order Taylor series expansion of logarithm function at the point $6(d_\mathcal{H}+d_\+T)$. Therefore,
    $$
    \begin{aligned}
    d & \le \frac{(d_\mathcal{H}+d_\+T)\log(6(d_\mathcal{H}+d_\+T))+2d_\mathcal{H}\log c - d_\+T(\log(d_\+T))}{\log2-1/6}\\ 
    & \le 
    2\left( (d_\mathcal{H}+d_\+T)\log(6(d_\mathcal{H}+d_\+T))+2d_\mathcal{H}\log c \right)
    \end{aligned}
    $$
    where the last step follows from $\log2-1/6<1/2$.
\end{proof}

% ------------------------------------------------------------------
% VC bound
% ------------------------------------------------------------------

\subsection{Proof of Corollary \ref{vc-example}}

The first two conclusions of corollary 4.4 are straightforward. We prove the last statement.

\begin{proof}
    Given $2p+3$ points in $\+X \times \+Y \times \+O$, without loss of generality, suppose there are at least $p+2$ points such that $o\ne y$.
    For these points, the value of annotation loss only depends on $S(w^\T x)$.
    For any $u\in\mathbb{R}$, the classifier 
    \[
        f_{h,T,u}(x,o)=\mathbb{1}\left\{\ell_\mathcal O(h(x),T,(x,o))>u\right\}
        =
        \mathbb{1}\left\{\log(S(w^\T x))<-u\right\}
    \]
    is a linear classifier with decision boundary $w^\T x=e^{-u}$. Since the VC dimension of hyperplanes of dimension $p$ is $p+1$, we know these linear classifiers cannot classify $p+2$ points arbitrarily. Therefore, the original $2p+3$ points cannot be classified arbitrarily, and we have $d_\+T\le 2p+2$.
\end{proof}

% ------------------------------------------------------------------
% separation
% ------------------------------------------------------------------

\subsection{Proof of Theorem \ref{separation}}

\begin{proof} 
    Denote the cross-entropy of two distributions $D_1$ and $D_2$ as $H(D_1, D_2)$ and the entropy of a distribution $D$ as $H(D)$. Let $\ell_\+O$ be the cross-entropy loss, for a fixed $x\in\+X$ we have that
    \[
        \begin{aligned}
        & \E_o[\ell_\mathcal{O}(h(x),T,(x,o))] - \E_o[\ell_\mathcal{O}(h_0(x),T_0,(x,o))] \\ 
        =\; & H((T_0(x))_{h_0(x)},(T(x))_{h(x)}) - H((T_0(x))_{h_0(x)},(T_0(x))_{h_0(x)}) \\ 
        =\; & H((T_0(x))_{h_0(x)},(T(x))_{h(x)}) - H((T_0(x))_{h_0(x)}) \\ 
        =\; & \KL{(T_0(x))_{h_0(x)}}{(T(x))_{h(x)}}
        \end{aligned}
    \]
    If $h(x)\ne h_0(x)$, then by the separation condition we have that 
    \[\KL{(T_0(x))_{h_0(x)}}{(T(x))_{h(x)}}\ge \gamma\]
    Also, if $h(x)= h_0(x)$, we have 
    \[\KL{(T_0(x))_{h_0(x)}}{(T(x))_{h_0(x)}}\ge \KL{(T_0(x))_{h_0(x)}}{(T_0(x))_{h_0(x)}} = 0\] 
    Therefore, for a fixed $h\in \+H$
    \[
    \begin{aligned}
        & R_\mathcal O(h\circ T)-\inf_{T\in\+T}R_\mathcal O(h_0 \circ T) \\ 
        =\; & \E_{x,o}[\ell_\mathcal{O}(h(x),T,(x,o))] - \E_{x,o}[\ell_\mathcal{O}(h_0(x),T_0,(x,o))] \\
        \ge\; & \P(h(x)\ne h_0(x)) \inf_{T,h(x)\ne h_0(x)} (\E_o[\ell_\mathcal{O}(h(x),T,(x,o))]-\E_o[\ell_\mathcal{O}(h_0(x),T,(x,o))]) \\ 
        & + \P(h(x)= h_0(x)) \inf_{T,h(x)= h_0(x)} (\E_o[\ell_\mathcal{O}(h(x),T,(x,o))]-\E_o[\ell_\mathcal{O}(h_0(x),T,(x,o))])\\ 
        \ge\; & \P(h(x)\ne h_0(x)) \inf_{T,h(x)\ne h_0(x)} (\E_o[\ell_\mathcal{O}(h(x),T,(x,o))]-\E_o[\ell_\mathcal{O}(h_0(x),T,(x,o))]) \\
        =\; & \P(h(x)\ne h_0(x)) \inf_{T,h(x)\ne h_0(x)} \KL{(T_0(x))_{h_0(x)}}{(T(x))_{h(x)}} \\ 
        \ge\; &  \P(h(x)\ne h_0(x)) \gamma \ge0
    \end{aligned}
    \]
    This shows the consistency condition \ref{Consistency}. Also, if $\P(h(x)\ne h_0(x))>0$, notice that $\P(h(x)\ne h_0(x))=R(h)$, we have
    \[
        \eta=\inf_{R(h)>0} \dfrac{R_\mathcal O(h\circ T)-\inf_{T\in\+T}R_\mathcal O(h_0 \circ T)}{R(h)} \ge \frac{\gamma R(h)}{R(h)}
        =\gamma>0
    \] 
    This shows the identifiability condition \ref{Identifiability}. 
    
    Moreover, if the condition \eqref{separation-condition} is not satisfied, by definition we have 
    \[
    \begin{aligned}
     \gamma =\; & \inf_{(x,i,j):p(x,y_i)>0,j\ne i,y_j\in\+H(x)} \KL{\mathcal{D}_i(x)}{\mathcal{D}_j(x)} \\ 
    =\; & \inf_{(x,i,j):p(x,y_i)>0,j\ne i,y_j\in\+H(x),D_i\in\+D_i(x),D_j\in\+D_j(x)} \KL{{D}_i}{{D}_j} \\ 
    =\; & 0
    \end{aligned}
    \]
    This condition implies that for any $k\in\mathbb{N^+}$, there exists a 5-tuple 
    \[\left(x^{(k)}, y^{(k)}_i, y^{(k)}_j, D^{(k)}_i(x^{(k)}), D^{(k)}_j(x^{(k)})\right)\in \+X\times \+Y\times\+Y \times \+D_\+O\times\+D_\+O\] 
    such that
    \begin{itemize}
        \item $p(x,y_i^{(k)})>0$
        \item $y^{(k)}_i\ne y^{(k)}_j$ 
        \item There is a $h_-\in\+H$ such that $h_-(x^{(k)})= y^{(k)}_j$
        \item $\KL{D^{(k)}_i(x^{(k)})}{D_j^{(k)}(x^{(k)})} < \frac1k$
    \end{itemize}
    Now, let $D^{(k)}_X$ be the point mass distribution with probability one to be $x^{(k)}$, i.e., $D^{(k)}_X(\{x^{(k)}\})=1$. Then, we have $h_0(x^{(k)})=y^{(k)}_i$ since $h_0$ has zero classification error. Also, let $T_0^{(k)}\in\+T$ be such that its $i^\text{th}$ row is $D^{(k)}_i$, and $T_-^{(k)}\in\+T$ be such that its $j^\text{th}$ row is $D^{(k)}_j$. We have
    $$
    \begin{aligned}
        \eta^{(k)}
        & = \underset{h \in \mathcal H: R(h)>0}{\inf} \dfrac{R_\mathcal O(h\circ T)-\inf_{T\in\+T}R_\mathcal O(h_0 \circ T)}{R(h)} \\
        & = \underset{h \in \mathcal H: R(h)>0}{\inf} \dfrac{R_\mathcal O(h\circ T)-R_\mathcal O(h_0 \circ T_0^{(k)})}{R(h)} \\ 
        & \le \dfrac{R_\mathcal O(h_-^{(k)}\circ T_-^{(k)})-R_\mathcal O(h_0 \circ T_0^{(k)})}{R(h_-^{(k)})}\\
        & \le \KL{D^{(k)}_i(x^{(k)})}{D_j^{(k)}(x^{(k)})}\\
        & \le \frac1k
    \end{aligned}
    $$
    Let $k\rightarrow \infty$ and the desired result follows.
\end{proof}

% ------------------------------------------------------------------
% concentration
% ------------------------------------------------------------------

\subsection{Proof of Proposition \ref{Robustness}}

\begin{proof}
First, for any $(x,y_i)\in\+X\times\+Y$ with $p(x,y_i)>0$ and $D_i\in\+D_i(x)$, $D_j\in\+D_j(x)$, by Pinsker's inequality, we have
\[
\begin{aligned}
    \KL{D_i}{D_j} 
    & \ge 2 \|D_i-D_j\|_{\text{TV}}^2 = \frac12\|D_i-D_j\|_{1}^2 \\ 
    & = \frac12 \left( \sum_{o\in\+O} |D_i(o)-D_j(o)|\right)^2 \\ 
    & \ge \frac12 (|D_i(S_i-S_j)-D_j(S_i-S_j)| +|D_i(S_j-S_i)-D_j(S_j-S_i)|)^2 \\
    & \ge \frac12 (D_i(S_i-S_j)-D_j(S_i-S_j) - D_i(S_j-S_i) + D_j(S_j-S_i))^2 \\
    & = \frac12 (D_i(S_i-S_j)-D_i(S_j-S_i) + D_j(S_j-S_i) -D_j(S_i-S_j))^2 \\
    & = \frac12 (D_i(S_i)-D_i(S_j) + D_j(S_j) -D_j(S_i))^2 \\
    & \ge \frac12(2\gamma_C)^2 = 2\gamma_C^2
\end{aligned} 
\]
where $D_i(\cdot)$ is the probability measure over $\+O$ defined by $D_i$, and $S_i-S_j$ is the set subtraction: $S_i-S_j\eqdef\{o:o\in S_i \land o\notin S_j\}\subset \+O$. Taking infimum on both sides of the inequality gives the first result. Another proof for this result can be found in the proof of Proposition 5.8.

Next, consider the annotation loss $\ell_\mathcal{O}(h(x),T,(x,o))=\1\{o\notin S_{h(x)}\}$ and its ERM. Then we have
\[
\begin{aligned}
    & \E_{x,o} [\ell_\mathcal{O}(h(x),T,(x,o))] - \E_{x,o}[\ell_\mathcal{O}(h_0(x),T,(x,o))] \\ 
    =\; & \P(o\notin S_{h(x)}) - \P(o\notin S_{h_0(x)}) \\ 
    \ge \; & \P(h(x)\ne h_0(x)) \inf_{x:h(x)\ne h_0(x)} \left( \P(o\in S_{h_0(x)}) - \P(o\in S_{h(x)})\right) \\
    \ge\; & \P(h(x)\ne h_0(x))\gamma_C = R(h)\gamma_C
\end{aligned} 
\]
Therefore,
\[
    \eta=\inf_{R(h)>0} \dfrac{R_\mathcal O(h\circ T)-\inf_{T\in\+T}R_\mathcal O(h_0 \circ T)}{R(h)} \ge \frac{\gamma_C R(h)}{R(h)}
    =\gamma_C>0
\] 
as claimed.
\end{proof}

\subsection{Proof of Proposition \ref{Evidence}}

\begin{proof}
    Since $\Phi_{ij}$ is Lipschitz, then for any $a,b\in\mathbb{R}^s$, we have
    \[ 
        |\Phi_{ij}(a)-\Phi_{ij}(b)| \le L_{ij} \|a-b\|_1
    \]
    Hence, given $(x,i,j)$ such that $p(x,y_i)>0,j\ne i$ and $y_j\in\+H(x)$, then for any $D_i\in\+D_i(x)$ and $D_j\in\+D_j(x)$, by Lipschitz property we have
    \[ 
        \|D_i-D_j\|_1 
        \ge \frac{1}{L_{ij}} |\Phi_{ij}(D_i)-\Phi_{ij}(D_j)|
        \ge \frac{\gamma_{ij}}{L_{ij}}
    \]
    Therefore, by Pinsker's inequality, we have
    \[ 
        \KL{D_i}{D_j} \ge \frac12\|D_i-D_j\|_1^2 \ge \frac12 \left(\frac{\gamma_{ij}}{L_{ij}}\right)^2
        \ge \frac12 \min_{i\ne j} \left(\frac{\gamma_{ij}}{L_{ij}}\right)^2
    \]
    Taking infimum on the left hand side of the inequality gives the desired result.

    In particular, if $\Phi$ represents the inner product with a fixed vector $u$, i.e., $\Phi(a)=\langle u,a\rangle$, then $\Phi$ is Lipschitz since for any $a,b\in\mathbb{R}^s$, by the H\"older's inequality, we have
    \[ 
    \begin{aligned}
        |\Phi(a)-\Phi(b)| & = |\langle u, a-b \rangle|
        \le \|u\|_\infty \|a-b\|_1
    \end{aligned}
    \]
    Therefore, we can bound the Lipschitz constant of $\Phi$ by $L\le \|u\|_\infty$.

    To recover the concentration condition, given sets $S_i\subset \+O (1\le i \le c)$, for any $i\ne j$, let 
    \[
        \Phi_{ij}(a)=\left\langle\sum_{k:o_k\in S_i} \hat e_k -\sum_{k:o_k\in S_j} \hat e_k ,a\right\rangle
    \]
    Then $\Phi_{ij}(D_i)=\P_{D_i}(O\in S_i)-\P_{D_i}(O\in S_j)$ and $\Phi_{ij}(D_j)=\P_{D_j}(O\in S_i)-\P_{D_j}(O\in S_j)$. The concentration condition \eqref{concentration-condition} implies that
    \[
        \inf_{p(x,y_i)>0, y_j\in\+H(x), D_i\in \+D_i(x), D_j\in \+D_j(x)}  
        \Phi_{ij}(D_i) - \Phi_{ij}(D_j)
        \ge 2\gamma_C >0
    \]
    % \[
    %     \inf_{T\in\+T, p(x,y_i)>0, y_j\in\+H(x)} \Phi_{ij}((T(x))_i) - 
    %     \sup_{T\in\+T, p(x,y_i)>0, y_j\in\+H(x)} \Phi_{ij}((T(x))_j)
    %     \ge 2\gamma_C >0
    % \]
    Moreover, since 
    $\left\|\sum_{k:o_k\in S_i} \hat e_k -\sum_{k:o_k\in S_j} \hat e_k\right\|_\infty=1$
    , the separation degree can be bounded by $\gamma \ge \frac12 \min_{i\ne j} \left({\gamma_{ij}}\right)^2 = 2\gamma_C^2$.
\end{proof}

\subsection{Proof of Proposition \ref{no-free-separation}}

\begin{proof}

    Given $D_i\in\+D_i(x)$ and $D_j\in\+D_j(x)$, write $D_i=\lambda D_{i1} + (1-\lambda) D_{i2}$
    and $D_j=\lambda D_{j1} + (1-\lambda) D_{j2}$, where 
    $D_{i1}\in\+D_{i1}(x)$,
    $D_{j1}\in\+D_{j1}(x)$,
    $D_{i2}\in\+D_{i2}(x)$,
    $D_{j2}\in\+D_{j2}(x)$.
    The summation $D_i=\lambda D_{i1} + (1-\lambda) D_{i2}$ means that we combine $D_{i1}$ and $D_{i2}$ as distributions over $\+O$ such that $D_i(o)=\lambda\1\{o\in\+O_1\}D_{i1}(o)+(1-\lambda)\1\{o\in\+O_2\}D_{i2}(o)$ for any $o\in\+O$.
    
    The first result basically follows from the convexity of KL-divergence: we have
    \begin{equation} \label{compare-KL}
    \begin{aligned}
        \KL{D_i}{D_j} 
        & = \KL{\lambda D_{i1} + (1-\lambda) D_{i2}}{\lambda D_{j1} + (1-\lambda) D_{j2}}\\ 
        & \le \lambda \KL{D_{i1}}{D_{j1}}
            +(1-\lambda)\KL{D_{i2}}{D_{j2}}
    \end{aligned}
    \end{equation}
    Hence,
    \[
        \begin{aligned}
            \lambda \KL{D_{i1}}{D_{j1}}
            +(1-\lambda)\KL{D_{i2}}{D_{j2}}
            \ge
            \inf_{x:p(x,y_i)>0,y_j\in\+H(x)}\KL{\+D_i}{\+D_j} = \gamma_{i\rightarrow j}
        \end{aligned}
    \]
    Take infimum again on the left hand side of the inequality, we have
    \[
        \lambda \gamma_{i\rightarrow j 1}
        + (1-\lambda) \gamma_{i\rightarrow j2} \ge 
        \gamma_{i\rightarrow j}
    \]
    More over, if $\+O_1\cap \+O_2=\emptyset$, then in \eqref{compare-KL}, we have 
    \[
        \begin{aligned}
            & \KL{\lambda D_{i1} + (1-\lambda) D_{i2}}{\lambda D_{j1} + (1-\lambda) D_{j2}} \\
            =\; & \sum_{o\in\+O} 
                (\lambda D_{i1}(o) + (1-\lambda) D_{i2}(o))
                \log\left(
                    \frac{\lambda D_{i1}(o) + (1-\lambda) D_{i2}(o)}
                    {\lambda D_{j1}(o) + (1-\lambda) D_{j2}(o)}
                \right) \\ 
            =\; & \sum_{o\in\+O_1}
                \lambda D_{i1}(o)
                \log\left(
                    \frac{\lambda D_{i1}(o)}
                    {\lambda D_{j1}(o)}
                \right)
                + \sum_{o\in\+O_2}
                (1-\lambda) D_{i2}(o)
                \log\left(
                    \frac{(1-\lambda) D_{i2}(o)}
                    {(1-\lambda) D_{j2}(o)}
                \right)\\ 
            =\; & \lambda \KL{D_{i1}}{D_{j1}}
            +(1-\lambda)\KL{D_{i2}}{D_{j2}}
        \end{aligned}
    \]
    Hence taking infimum on both sides gives
    $\lambda \gamma_{i\rightarrow j 1}
    + (1-\lambda) \gamma_{i\rightarrow j2} = 
    \gamma_{i\rightarrow j}$.

    The above discussion shows that if $\gamma_{i\rightarrow j}>0$, then one of $\gamma_{i\rightarrow j 1}$ and $\gamma_{i\rightarrow j2}$ must be positive.

\subsection{Remark \ref{defining-intersection}}
        
    Finally, we show by a simple example that if $\+O_1\cap \+O_2\ne\emptyset$, then even if both $\lambda \gamma_{i\rightarrow j 1}$ and $\gamma_{i\rightarrow j2}$ are positive, we can still have $\lambda \gamma_{i\rightarrow j}=0$. 
    
    Consider a binary classification ($\mathcal{Y}=\{\pm1\}$) with two noisy annotations (crowdsourcing with two annotators) $O_1$ and $O_2$.
    Suppose the transitions of the two annotations are known to the learner and are given by constant matrices
    \[
        T_1(x)\equiv
        \begin{bmatrix}
            0.6 & 0.4 \\ 
            0.4 & 0.6
        \end{bmatrix}
        \;\text{and}\;\;
        T_2(x)\equiv
        \begin{bmatrix}
            0.4 & 0.6 \\ 
            0.6 & 0.4
        \end{bmatrix}
    \]
    Then, individually, both the annotations can ensure separation.
    However, suppose $\lambda=1/2$, then in this case, if the annotations are mixed (i.e., the learner do not distinguish the annotations of different annotators, and hence $\+O=\+O_1\cup\+O_2=\+Y$), then for any $x,y$,
    \[
        \P(O=y|x,y)=\lambda\P(O_1=y|x,y) + (1-\lambda)\P(O_2=y|x,y)=1/2\]
    Here we used the condition that $\1\{O=O_k\}$ is independent with $X$. Now, it is not possible to learn $Y$ from the observation of $O$ since $O$ is simply a random noise that is independent of $Y$.
\end{proof}

\bibliographystyle{plainnat}
\bibliography{ccg, new, cited}

\end{document}